\documentclass[journal,twoside,web]{ieeecolor}
\usepackage{generic}
\usepackage{cite}
\DeclareSymbolFont{timescal}{OMS}{ztmcm}{m}{n}
\DeclareSymbolFontAlphabet{\mathcal}{timescal}
\usepackage{amsmath,amssymb,amsfonts}
\usepackage{mathtools}
\usepackage[mathscr]{eucal}
\usepackage{upgreek}
\DeclareMathAlphabet{\mathpzc}{OT1}{pzc}{m}{it}
\newtheorem{assumption}{Assumption}
\newtheorem{theorem}{Theorem}
\usepackage{algorithmic}
\usepackage{graphicx}
\usepackage{algorithm,algorithmic}
\usepackage{hyperref}
\hypersetup{hidelinks=true}
\usepackage{textcomp}

\def\BibTeX{{\rm B\kern-.05em{\sc i\kern-.025em b}\kern-.08em
    T\kern-.1667em\lower.7ex\hbox{E}\kern-.125emX}}
\begin{document}
\title{ Learning Neural Feedback Linearization for Data-driven Systems via Augmented Lagrangian}
\author{Lakshmi Priya P. K. , Andreas Schwung
\thanks{Lakshmi Priya P. K. and Andreas Schwung are with the Department of Automation Technology and Learning Systems, South Westphalia University of Applied Sciences,Lübecker Ring 2, Soest, 59494, Germany (pondicherrykumar.lakshmipriya@fh-swf.de; schwung.andreas@fh-swf.de). }
}
\maketitle
\begin{abstract}
The paper proposes a novel data-driven framework for designing and training a feedback linearizing controller by explicitly incorporating relative degree based conditions into the learning process. This enables the conventional feedback controller components to be replaced by neural Lie derivatives, thereby facilitating a fully data-driven feedback linearization framework. Furthermore, practical closed-loop stability is established by deriving sufficient conditions under which bounded identification errors lead to bounded tracking errors. The derived theoretical results are validated through their application to an armature controlled DC motor.
\end{abstract}

\begin{IEEEkeywords}
Closed-loop stability, data-driven control systems, feedback-linearization, neural Lie derivatives
\end{IEEEkeywords}

\section{Introduction}
\label{sec:introduction}

Feedback linearization is a well established nonlinear control technique that converts a nonlinear dynamical system into an equivalent linear representation through a suitable change of coordinates and a nonlinear control law. This allows for systematic controller design, stability analysis, and performance evaluation using established linear systems theory. Consequently, feedback linearization has been widely applied to nonlinear systems characterized by complex dynamical behavior \cite{b1}. However, its practical implementation relies on accurate knowledge of the system model, which may not be readily available for complex intricate systems.

\par Data-driven approaches address this challenge by using measurement data directly for controller synthesis and analysis \cite{k0}. While this reduces reliance on explicit physical equations, it introduces additional challenges in experiment design, model structure selection, parameter estimation, and interpretation \cite{s0,s1}. Early pioneering research in data-driven control (DDC) can be found in \cite{z1,h0,a2,s4}. 

\par Due to recent major advancements in machine learning, data-driven modeling have gained considerable attention and widespread adoption \cite{m1,v1}. Neural networks provide an effective framework for data-driven modeling as flexible function approximators capable of capturing complex nonlinear system dynamics directly from measured data \cite{d1}. Their differentiable structure also enables to jointly learn the system representation and construct controller directly from data \cite{k1, s2, u1, y1}. The increasing availability of measured process data has also motivated data-driven feedback linearization approaches \cite{n1}. These methods retain advantages of feedback linearization while extending its applicability to systems with unknown or difficult to derive governing dynamics.
 
\par Since analytical tools in nonlinear control theory are grounded in differential-geometric concepts, the corresponding feedback linearization requires the learned neural network model to preserve the structural properties necessary for enforcing relative degree condition. However, an approach which enforces the relative degree condition during training, and links learning errors to closed-loop tracking performance is missing in the literature.  Existing data-driven feedback linearization methods either learn predictive models without explicitly incorporating the geometric structure required for feedback linearization or recover linearization transformations from predefined dictionaries using sparse regression \cite{b2,f2,h3}. Furthermore, stability of the resulting closed-loop  controller is generally not addressed, particularly in the presence of bounded identification errors.

We overcome these limitations by introducing a data-driven feedback linearization framework using multilayer perceptron (MLP) neural network. The differential geometric structure are incorporated directly into the learning process through an Augmented-Lagrangian formulation. The feedback-linearizing controller is directly defined by the networks, promoting consistency between model and structural requirements of the controller design. A practical stability result is established showing that bounded identification errors lead to ultimately bounded tracking errors, providing a closed-loop stability characterization for the proposed learning-based control.

\subsubsection{Related work}
\par The existing literature can be broadly divided into direct and indirect control design. Indirect approaches first identify a system model and subsequently use it for controller design \cite{b2}. Sparse identification and Koopman-based approaches establish structured data-driven representations of nonlinear dynamics, providing a systematic foundation for subsequent control synthesis \cite{h3}. Koopman generator least-squares formulations recover linearizing state and control transformations from prescribed function dictionaries \cite{g2}. The integration of learning techniques with feedback linearization has also been investigated using Gaussian processes \cite{u1}. Robustness of sparse identification with feedback linearization and an additional observer for finite-time compensation of modeling errors and disturbances is studied in \cite{r1}. The authors of \cite{y1} use neural networks to approximate the unknown system online which are used to construct the feedback-linearizing controller which ensures uniform ultimate boundedness of all closed-loop signals via Lyapunov analysis.

\par In direct data-driven control, the controller is obtained directly from measured data without first identifying an explicit model of the plant \cite{p2,m2}. This approach reduces the intermediate modeling step and can be useful when the main objective is controller synthesis rather than model reconstruction. It supports the idea that system trajectories can be represented directly from data under persistency of excitation \cite{w2}.  Neural networks were initially used as nonlinear identifiers and adaptive control models, with subsequent extensions to architectures tailored towards controller synthesis \cite{t1,k1}. Representative approaches include quasi-ARX models and control-affine neural NARX models \cite{h2,x1}. In \cite{s3}, the authors employed a neural network to derive state and control transformations, while \cite{w1} introduced an approach based on reinforcement learning. Sampled-data stabilization has been established for single-input single-output (SISO) systems without explicit model identification \cite{f2}, whereas behavioral formulations provide data-based trajectory representations for multiple-input multiple-output (MIMO) feedback-linearizable systems under basis-function approximation and noisy output measurements \cite{a1}. A data-driven dynamic output-feedback controller was developed to robustly stabilize linear MIMO systems using only noisy input–output measurements \cite{l2}. Data-driven output-feedback stabilization under measurement noise and neglected nonlinearities has been studied through auxiliary input–output representations and data-dependent linear matrix inequalities, although not within a feedback-linearization framework \cite{d2}. A related result studies data-driven feedback linearization using stacked regression with functional dictionaries \cite{l1}.
\subsubsection{Contributions}
Current literature on data-driven control approaches either learn system representations without explicitly preserving the differential geometric structure required for feedback linearization or derive linearizing transformations while closed-loop behavior in the presence of identification errors remains insufficiently characterized. To address these limitations, we make the following contributions:
\begin{itemize}
\item We propose a data-driven feedback-linearization architecture for unknown nonlinear systems by using neural networks to approximate the system dynamics.
\item We formulate a constrained optimization framework to enforce the required relative degree condition during training which is reformulated as an unconstrained optimization problem using augmented Lagrangian multiplier technique.
\item A practical closed-loop stability result is established showing that bounded identification error in the learned vector fields leads to ultimately bounded tracking error and the learning feedback-linearizing control architecture remains stable.
\item The proposed learning-based feedback-linearization framework and the resulting closed-loop stability is validated on an armature-controlled DC motor.
\end{itemize}
The article is organized as follows: Section 2 reviews the fundamental concepts and introduces the neural identification framework. Section 3 presents the main theoretical results, i.e. the data-driven feedback linearizing control using an augmented Lagrangian framework. Section 4 establishes practical closed-loop stability. Section 5 applies the approach to an armature-controlled DC motor. Section 6 concludes the paper.

\section{Basic concepts and Preliminaries}
In this section, we recall the foundational concepts underlying in identification and feedback control of model free systems using neural networks. In particular, we discuss how a MLP architecture can be constructed to model a nonlinear control system solely from the process data. 

\subsection{Data Collection and Model Construction}
Consider the following nonlinear control affine system 
\begin{align} \label{1}
    \dot{\mathrm{x}}(\mathpzc{t}) &= \mathsf{f}(\mathrm{x}(\mathpzc{t})) +\mathsf{g}(\mathrm{x}(\mathpzc{t}))\mathsf{u}(\mathpzc{t}), \nonumber \\ 
\mathsf{y}(\mathpzc{t}) &= \mathsf{c}(\mathrm{x}(\mathpzc{t})).
\end{align}
The state $\mathrm{x}(\mathpzc{t}) \in \mathscr{R}^{n},$ $\mathsf{u}(\mathpzc{t})\in \mathscr{R}^{n_u}$ is the control input and $\mathsf{y} (\mathpzc{t})\in \mathscr{R}^{n_y}$ is the output of the system. The system is considered in the SISO setting, with $n_y=n_{u}=1$, and is assumed to have a well defined, known relative degree $r=n$. We note, that the methodology can be generalized to MIMO systems while extension to the case with relative degree $(r=\delta<n)$ for unknown systems has been discussed in \cite{l1} using a sparse regression approach, which can be extended to the setting presented in this paper. A finite number of $m$ data samples of  $\mathrm{x}(\mathpzc{t}),$  $\mathsf{u}(\mathpzc{t})$ and  $\mathsf{y}(\mathpzc{t})$ are collected and organized in a standard form \cite{b2}. The collected process data are represented in the form $\mathrm{X} = \left[\mathrm{x}(\mathpzc{t}_{1}), \ldots,
\mathrm{x}(\mathpzc{t}_{m}) \right]^{T} \in \mathscr{R}^{m\times n},
\dot{\mathrm{X}} = \left[\dot{\mathrm{x}}(\mathpzc{t}_{1}),\ldots,\dot{\mathrm{x}}(\mathpzc{t}_{m})\right]^{T} \in \mathscr{R}^{m\times n},
\mathsf{U} = \left[\mathsf{u}(\mathpzc{t}_{1}), \ldots, \mathsf{u}(\mathpzc{t}_{m}) \right]^{T} \in \mathscr{R}^{m\times 1}, \mathsf{Y}=\left[\mathsf{y}(\mathpzc{t}_{1}), \ldots, \mathsf{y}(\mathpzc{t}_{m}) \right]^{T} \in \mathscr{R}^{m\times 1}$. If direct measurements of the derivatives are not readily available, they may be approximated numerically from the observed state data \cite{k3}.
The quality of the collected data plays a critical role in the predictive performance of the learned model. The data set is assumed to be persistently exciting and to provide sufficient coverage of the considered operating region for reliable identification of the system dynamics. Having constructed the dataset, the subsequent step is to develop the neural network architecture for the considered system.

\subsection{Neural network architecture}

We present the MLP parameterized approximations  $\mathsf{f}_{\upphi}$, $\mathsf{g}_{\upphi}$ and $\mathsf{c}_{\upphi}$ using sampled data. By learning directly from data, the network parameters capture the underlying structure of the system which enables the learned representation to approximate the observed state trajectories of the true system. 
Let
$\mathsf{f}_{\upphi}:\Omega \rightarrow \mathscr{R}^n, \,  \mathsf{g}_{\upphi}:\Omega \rightarrow \mathscr{R}^{n }, \, $ and  $\mathsf{c}_{\upphi}: \Omega  \rightarrow \mathscr{R}^{n_y},\, \Omega \subset \mathscr{R}^n$ is an open subset, then 
\begin{align} \label{6} 
   \hat{ \dot{\mathrm{x}}}(\mathpzc{t})&= \mathsf{f}_{\upphi}(\mathrm{x}(\mathpzc{t})) +\mathsf{g}_{\upphi}(\mathrm{x}(\mathpzc{t}))\mathsf{u}(\mathpzc{t}), \nonumber \\ 
 \hat{\mathsf{y}}(\mathpzc{t})  &= \mathsf{c}_{\upphi}(\mathrm{x}(\mathpzc{t})).
 \end{align}
For  $k=1,2,\cdots,m, \,$ $\hat{ \dot{\mathrm{x}}}_{k}= \mathsf{f}_{\upphi}(\mathrm{x}_{k}) +\mathsf{g}_{\upphi}(\mathrm{x}_{k})\mathsf{u}_{k} \, \in \mathscr{R}^n$ defines the prediction dynamics per sample.
Let $n_{in}$, $n_h$, and $n_o$ denote the number of neurons in the input, hidden, and output layers, respectively. For an $l$-layer MLP let $\mathscr{W}^{(1)},\mathscr{V}^{(1)} \in\mathscr{R}^{n_h\times n_{in}}, \mathscr{W}^{(i)}, \mathscr{V}^{(i)} \in\mathscr{R}^{n_h\times n_h},\; i=2,\ldots,l-1, \mathscr{W}^{(l)}\in\mathscr{R}^{n_o\times n_h}, \mathscr{V}^{(l)} \in \mathscr{R}^{1 \times n_h}, b^{(i)}, a^{(i)} \in\mathscr{R}^{n_h},\; i=1,\ldots,l-1, 
b^{(l)}\in\mathscr{R}^{n_o}, a^{(l)}\in \mathscr{R}$ .
Then for $\mathrm{x}\in\mathscr{R}^{n}$, the drift, input-vector-filed and the output network are represented as
\begin{align*}
\mathsf{f}_{\upphi}(\mathrm{x})
&\!=\!
\mathscr{W}_{\mathsf{f}}^{(l)}\!
\sigma\!\left(
\mathscr{W}_{\mathsf{f}}^{(l-1)}\!
\sigma\!\left(\!
\cdots\!
\sigma\!\left(\!
\mathscr{W}_{\mathsf{f}}^{(1)}\!\mathrm{x}\!
+\!b_{\mathsf{f}}^{(1)}\!
\right)\!
\cdots\!
\right)
\!+\!b_{\mathsf{f}}^{(l-1)}\!
\right)
\!+\!b_{\mathsf{f}}^{(l)}\!, \\
\mathsf{g}_{\upphi}(\mathrm{x})
&\!=\!
\mathscr{W}_{\mathsf{g}}^{(l)}
\!\sigma\!\left(
\mathscr{W}_{\mathsf{g}}^{(l-1)}\!
\sigma\!\left(\!
\cdots\!
\sigma\!\left(\!
\mathscr{W}_{\mathsf{g}}^{(1)}\!\mathrm{x}
\!+\!b_{\mathsf{g}}^{(1)}\!
\right)
\!\cdots\!
\right)
\!+\!b_{\mathsf{g}}^{(l-1)}\!
\right)
\!+\!b_{\mathsf{g}}^{(l)}\!,\\
\mathsf{c}_{\upphi}(\mathrm{x})
&\!=\!
\mathscr{V}^{(l)}
\!\sigma\!\left(
\mathscr{V}^{(l-1)}\!
\sigma\!\left(
\!\cdots\!
\sigma\!\left(
\mathscr{V}^{(1)}\mathrm{x}
\!+\!a^{(1)}
\right)
\!\cdots\!
\right)
\!+\!a^{(l-1)}
\right)
\!+\!a^{(l)}\!.
\end{align*}

Note that forward equations of $\mathsf{c}_{\upphi}(\mathrm{x})$ through layers  can be equivalently represented by
\begin{align*}
z_{\mathsf c}^{(1)}(\mathrm{x})
&\!=\!
\mathscr{V}^{(1)}\mathrm{x}\!+\!a^{(1)}
\!\in\!\mathscr{R}^{n_h},
\,
h_{\mathsf c}^{(1)}(\mathrm{x})
\!=\!
\sigma\!\left(\!z_{\mathsf c}^{(1)}\!(\mathrm{x})\!\right)
\!\!\in\!\mathscr{R}^{n_h},
\\
z_{\mathsf c}^{(i)}(\mathrm{x})
&\!=\!
\mathscr{V}^{(i)}h_{\mathsf c}^{(i-1)}\!(\mathrm{x})
\!+\!a^{(i)}
\!\!\in\!\mathscr{R}^{n_h},
\,
h_{\mathsf c}^{(i)}\!(\mathrm{x})
\!=\!
\sigma\!\left(\!z_{\mathsf c}^{(i)}(\mathrm{x})\!\right)
\!\!\in\! \mathscr{R}^{n_h},
\,
\\
i&\!=\!2,\ldots,l\!-\!1, \quad
\mathsf{c}_{\upphi}(\mathrm{x})
\!=\! \mathscr{V}^{(l)}h_{\mathsf c}^{(l-1)}\!(\mathrm{x})
\!+\!a^{(l)}\!\!
\in\!\mathscr{R}.
\end{align*}

Having defined the neural network representation of the unknown system we define the loss function of \eqref{6} as the empirical mean squared error between the target values and the network predictions: 
\begin{align} \label{10}
    \mathscr{L}(\upphi) &\!=\!\left\| \dot{\mathrm{X}}\!-\!\mathsf{f}_{\upphi}(\mathrm{X}) \!-\! \mathsf{g}_{\upphi}(\mathrm{X})\odot\mathsf{U}\right\|_{2}^{2} \!+\! \lambda \left\|\mathsf{Y}\!-\!\mathsf{c}_{\upphi}(\mathrm{X})\right\|_{2}^{2},
\end{align}
where $\odot$ denotes the Hadamard product. The parameter $\lambda>0,$ is a weighting factor, determines the relative importance of the output prediction error compared to the state prediction error in the combined loss function. For data-driven feedback linearization the optimization \eqref{10} is not sufficient to ensure that the learned model satisfies the structural conditions required for feedback linearization. Therefore, additional constraints have to be introduced to enforce the relative-degree conditions.

\subsection{Feedback linearization}
This section introduces the preliminary concepts of feedback linearization \cite{i1}. Feedback linearization transforms the original nonlinear system using a smooth and invertible coordinate transformation (diffeomorphism)
$\mathpzc{q}:\mathscr{R}^{n}\rightarrow \mathscr{R}^{n}$ by means of Lie derivatives $\mathcal{L}_{\mathsf{a}}\mathsf{b}(\mathrm{x})=\frac{\partial\mathsf{b}(\mathrm{x})}{\partial \mathrm{x}}\,\mathsf{a}(\mathrm{x}) $, yielding a new coordinate frame with linear closed-loop dynamics. 
The integer $r,$  $ 1\leq r \leq n $, for which 
$
\mathcal{L}_{\mathsf{g}} \mathcal{L}_{\mathsf{f}}^{p-1}\,\mathsf{c}(\mathrm{x})=0,\, p=1,\ldots,r-1,
$
and
$
\mathcal{L}_{\mathsf{g}} \mathcal{L}_{\mathsf{f}}^{r-1}\,\mathsf{c}(\mathrm{x})\neq 0,
$ hold is defined as the relative degree of the system. Within the transformed framework, the feedback control is defined as
\begin{align} \label{dif1.1}
    \mathsf{u}
    &\!=\!
    -\dfrac{
        \mathcal{L}_{\mathsf{f}}^{n}\mathsf{c}(\mathrm{x})
        \!+\!k_{n-1}\mathcal{L}_{\mathsf{f}}^{n-1}\mathsf{c}(\mathrm{x})
        \!+\!\cdots
        \!+\!k_{0}\mathsf{c}(\mathrm{x})\!-\!V\mathpzc{w}
    }{
        \mathcal{L}_{\mathsf{g}}\mathcal{L}_{\mathsf{f}}^{\,n-1}\mathsf{c}(\mathrm{x})
    }  
\end{align}
gives the closed-loop input-output dynamics, where the auxiliary input is introduced so that
\[
V \mathpzc{w}
=
\mathsf{y}^{(n)}
+k_{n-1}\mathsf{y}^{(n-1)}
+\cdots
+k_{1}\dot{\mathsf{y}}
+k_{0}\mathsf{y},
\]
$\mathpzc{w}$ is the new external input to the linearized system and $V$ is a prefilter gain. In the proposed framework, conventional control \eqref{dif1.1} is replaced with neural Lie derivative components, enabling the design and training of data-driven controller. 

\section {Data-driven feedback linearization for Composite neural network}

We now formulates the data-driven feedback linearization which is subdivided into a learning stage and a control stage, see Fig.~\ref{fig:pipeline}. In the learning stage, process data collected from the nonlinear plant is used to train MLP representations of the unknown system, minimize the discrepancy between learned model and measured plant data. In parallel, structural conditions to enforce the relative degree are imposed as constraints on the learning problem. To handle the resulting nonconvex constraints, we use an augmented Lagrangian formulation, in which the constraints are incorporated into the learning objective through Lagrange multipliers. In the control stage, the learned quantities are used to construct data-driven feedback-linearizing controller. Proportional–integral action is incorporated to improve the closed-loop tracking performance.
\begin{figure}[t]
    \centering
    \includegraphics[width=0.45\textwidth]{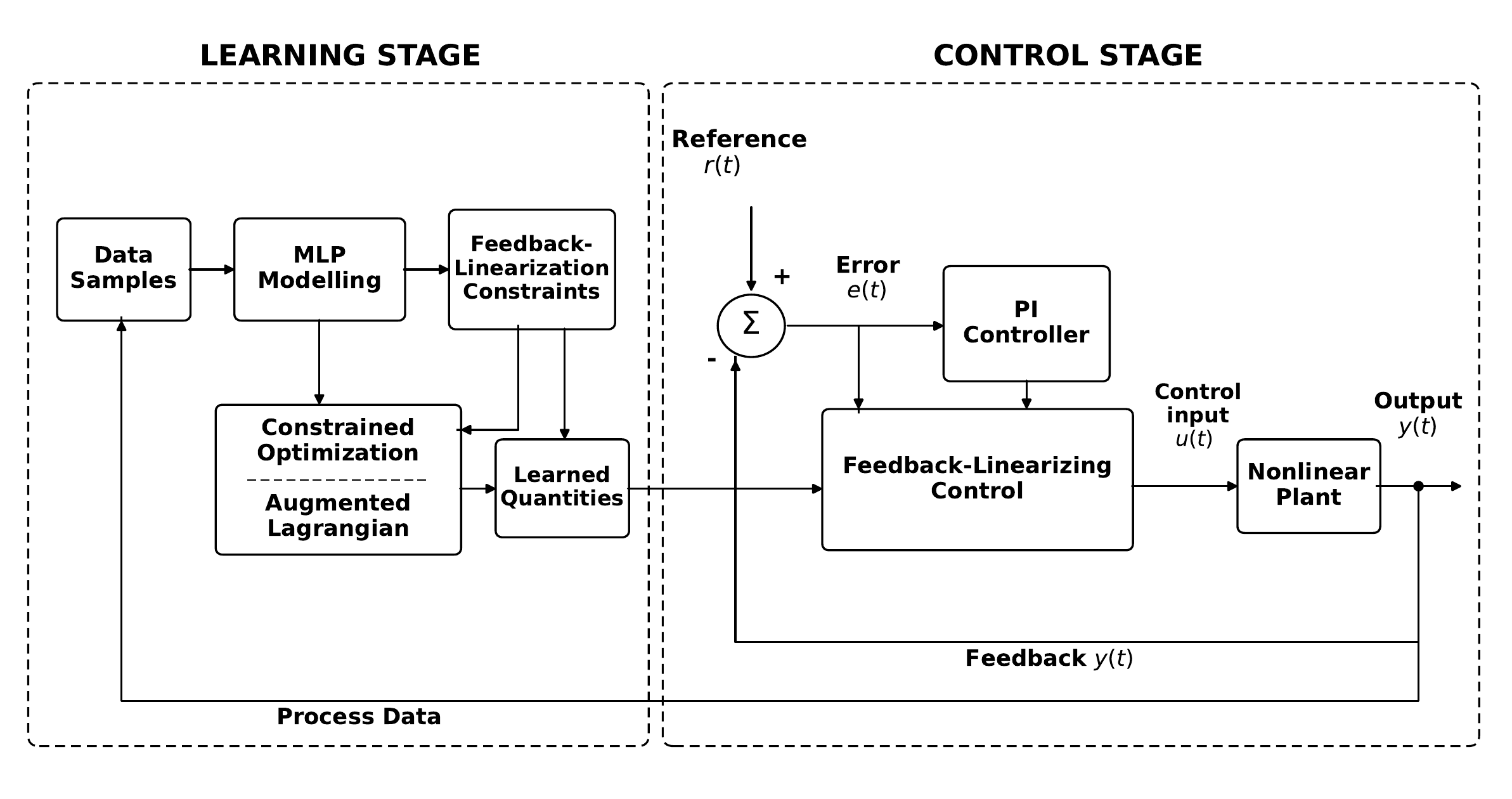}
    \caption{Schematic of the proposed data-driven feedback-linearization framework, comprising a learning stage and a control stage.}
    \label{fig:pipeline}
\end{figure}

\subsection{Feedback-linearization framework}
\par To enable smooth computation of the neural Lie derivatives required for data-driven feedback-linearization, the following assumption is imposed on the activation function: 
\begin{assumption}
   The activation functions satisfy $\sigma\in \mathpzc{C}^{\infty}$.
\end{assumption}
In this work, the sigmoid activation function is chosen to satisfy this regularity condition. We now state the main result:
\begin{theorem} \label{main}
For $\mathrm{x}\in \mathscr{R}^{n}$, consider the data-driven neural network representation of the nonlinear control system $(\ref{6})$ with full relative degree $r=n$. Then, the corresponding learning-based neural feedback controller is given by:
\begin{align}\label{c1}
    \mathsf{u}(\mathrm{x})\!=\! - \dfrac{\theta_{n}(\mathrm{x})+k_{n-1}\theta_{n-1}(\mathrm{x})+\cdots k_{0}\theta_{0}(\mathrm{x})-V\mathrm{w}}{\theta_{n-1}^{'}(\mathrm{x})\mathsf{g}_{\upphi}(\mathrm{x})}.
\end{align}
The function $\theta_{p}(\mathrm{x}),  p=0,1,\cdots,n$, is recursively defined by
\[
\theta_{0}(\mathrm{x})=\mathsf{c}_{\upphi}(\mathrm{x})= \mathscr{V}^{(l)}h_{\mathsf{c}}^{(l-1)}
+a^{(l)}, \hspace{0.1cm}  \theta_{p}(\mathrm{x})=\theta_{p-1}^{'}(\mathrm{x})\mathsf{f}_{\upphi}(\mathrm{x}),
\]
where $\theta_{p}^{'} (\mathrm{x})=\dfrac{\partial \theta_{p} (\mathrm{x}) }{\partial \mathrm{x}}$ denotes the Jacobian and the controller components are determined by the following constrained optimization problem:
\begin{align} \label{3.2}
     &\min_{\upphi}   \left\| \dot{\mathrm{X}}-\mathsf{f}_{\upphi}(\mathrm{X}) - \mathsf{g}_{\upphi}(\mathrm{X})\odot\mathsf{U}\right\|_{2}^{2} + \lambda \left\|\mathsf{Y}-\mathsf{c}_{\upphi}(\mathrm{X})\right\|_{2}^{2},\nonumber\\
      & \text{subject to}\\
     &\theta_{p}^{'} (\mathrm{x}_{k}) \, \mathsf{g_{\upphi}}(\mathrm{x}_{k})\!=\!0, \, \forall \, p \!\in\! \{0,\ldots,n\!-\!2\},  \forall \, k \!\in\! \{1,\ldots,m \}\nonumber.
\end{align}
\end{theorem}
 \begin{proof}
        To establish that system (\ref{6}) is feedback linearizable, it is sufficient to verify that the learned neural representation satisfies the required relative-degree conditions. The proof proceeds by first computing the gradient of the output neural map, which is then used to derive the corresponding neural Lie derivatives characterizing the relative-degree conditions. These are subsequently incorporated as constraints into the learning procedure. Once these conditions are enforced, the resulting learned quantities are used to construct the data-driven feedback-linearizing control law.
        
        Since the governing dynamics of the system are expressed as a composition of nested nonlinear transformation, the state $\mathrm{x}$ is embedded within the layers of the MLP. Consequently, the gradient must be calculated by propagating into the network’s internal chain of weights and activations. For  the output function
\[
\mathsf{c}_{\upphi}(\mathrm{x})
=
\mathscr{V}^{(l)}h_{\mathsf{c}}^{(l-1)}(\mathrm{x})
+a^{(l)},
\]
the gradient computed using the chain rule is 
\begin{align}\label{3.4}
\frac{\partial \mathsf{c}_{\upphi}(\mathrm{x})}
{\partial \mathrm{x}}
&=
\frac{\partial \mathsf{c}_{\upphi}(\mathrm{x})}
{\partial h_{\mathsf{c}}^{(l-1)}(\mathrm{x})}
\frac{\partial h_{\mathsf{c}}^{(l-1)}(\mathrm{x})}
{\partial z_{\mathsf{c}}^{(l-1)}(\mathrm{x})}
\frac{\partial z_{\mathsf{c}}^{(l-1)}(\mathrm{x})}
{\partial h_{\mathsf{c}}^{(l-2)}(\mathrm{x})}
\cdots \nonumber \\
& \quad \times \frac{\partial z_{\mathsf{c}}^{(2)}(\mathrm{x})}
{\partial h_{\mathsf{c}}^{(1)}(\mathrm{x})}
\frac{\partial h_{\mathsf{c}}^{(1)}(\mathrm{x})}
{\partial z_{\mathsf{c}}^{(1)}(\mathrm{x})}
\frac{\partial z_{\mathsf{c}}^{(1)}(\mathrm{x})}
{\partial \mathrm{x}}
 \\
&=
\mathscr{V}^{(l)}
\sigma'\!(z_{\mathsf{c}}^{(l-1)}(\mathrm{x}))
\mathscr{V}^{(l-1)}
\cdots
\mathscr{V}^{(2)}
\sigma'\!(z_{\mathsf{c}}^{(1)}(\mathrm{x}))
\mathscr{V}^{(1)}.\nonumber
\end{align}
For $z_{\mathsf{c}}^{(i)}(\mathrm{x}) \in \mathscr{R}^{n_{h}}$, the sigmoid derivative of $j$-th neuron in the $i$-th layer is defined as
\begin{align*}
s_{\mathsf{c},j}^{(i)}(\mathrm{x})
\!=\! \sigma^{'}\!(z_{\mathsf{c},j}^{(i)}(\mathrm{x}))\!=\!
\sigma\!(z_{\mathsf{c},j}^{(i)}(\mathrm{x}))
[
1\!-\!\sigma\!(z_{\mathsf{c},j}^{(i)}(\mathrm{x}))
],
j\!=\!1,\ldots,n_h.  
\end{align*}
Then we have for $i=1,\cdots,l-1,$ 
\begin{align}
\sigma'\! (z_{\mathsf{c}}^{(i)}(\mathrm{x}))
\!=\!
[
s_{\mathsf{c},1}^{(i)}(\mathrm{x}),
s_{\mathsf{c},2}^{(i)}(\mathrm{x}),
\ldots,
s_{\mathsf{c},n_h}^{(i)}(\mathrm{x})
]^{T}, \nonumber\\
S_{\mathsf{c}}^{(i)}(\mathrm{x})
\!=\!
\frac{\partial h_{\mathsf{c}}^{(i)}(\mathrm{x})}
{\partial z_{\mathsf{c}}^{(i)}(\mathrm{x})}\!=\!\!
\begin{bmatrix} \label{3.5}
s_{\mathsf{c},1}^{(i)}(\mathrm{x})\!\!\!\! & \!\cdots\! & \!0\\
0 & \cdots & 0\\
\vdots & \ddots & \vdots\\
0 & \cdots & \!\!\!\!s_{\mathsf{c},n_h}^{(i)}(\mathrm{x})
\end{bmatrix}\! \!\in\! \mathscr{R}^{n_{h}} \!\!\times\! \mathscr{R}^{n_{h}}\!.
\end{align}
Substituting \eqref{3.5} in \eqref{3.4} the preceding chain-rule expression can be rewritten in compact form
\[
\frac{\partial \mathsf{c}_{\upphi}(\mathrm{x})}
{\partial \mathrm{x}}
=
\mathscr{V}^{(l)}
S_{\mathsf{c}}^{(l-1)}(\mathrm{x})
\mathscr{V}^{(l-1)}
\cdots
\mathscr{V}^{(2)}
S_{\mathsf{c}}^{(1)}(\mathrm{x})
\mathscr{V}^{(1)}.
\]
Therefore, the derivative of the output network is given by
\begin{align*}
(\nabla \mathsf{c}_{\upphi}(\mathrm{x}))^{T}=   
(\mathscr{V}^{(l)}
S_{\mathsf{c}}^{(l-1)}(\mathrm{x})
\mathscr{V}^{(l-1)}
\cdots
\mathscr{V}^{(2)}
S_{\mathsf{c}}^{(1)}(\mathrm{x})
\mathscr{V}^{(1)})\hat{\mathrm{e}},
\end{align*}
where $\hat{\mathrm{e}} \!\in\! \mathscr{R}^{1 \times n}$ denotes the canonical unit vector and $\hat{\mathrm{e}}$ embeds the output scalar derivative to the corresponding component of the $n$ dimensional state gradient. We compute the iterated Lie derivatives 
\begin{align*}
    \mathcal{L}^{1}_{\mathsf{f}_{\upphi}}\mathsf{c}_{\upphi}(\mathrm{x})&= (\mathscr{V}^{(l)}
S_{\mathsf{c}}^{(l-1)}(\mathrm{x})
\cdots
S_{\mathsf{c}}^{(1)}(\mathrm{x})
\mathscr{V}^{(1)})\nonumber \\
& \quad \times \hat{\mathrm{e}} (\mathscr{W}_{\mathsf{f}}^{(l)}h_{\mathsf{f}}^{(l-1)}(\mathrm{x})+b_{\mathsf{f}}^{(l)}) \nonumber \\
&=\theta_{0}^{'}(\mathrm{x})\mathsf{f}_{\upphi}(\mathrm{x})=\theta_{1}(\mathrm{x}),  \nonumber\\
     &\vdots \nonumber\\
       \mathcal{L}^{n-1}_{\mathsf{f}_{\upphi}}\mathsf{c}_{\upphi}(\mathrm{x}) &=\theta_{n-1}(\mathrm{x}).
\end{align*}
Next we proceed to examine the conditions $\mathcal{L}_{\mathsf{g}_{\upphi}} \mathcal{L}_{\mathsf{f}_{\upphi}}^{p}\,\mathsf{c}_{\upphi}(x)=0,\, p=0,\ldots,n-2,$ required for the system \eqref{6} to possess relative degree $r=n.$ Therefore we  enforce the following equality constraints represented by the neural Lie derivatives
\begin{align*}
    \mathcal{L}_{\mathsf{g}_{\upphi}} \mathcal{L}^{0}_{\mathsf{f}_{\upphi}}\mathsf{c}_{\upphi}(\mathrm{x})&=\theta^{'}_{0}(\mathrm{x})\, \mathsf{g}_{\upphi}(\mathrm{x})=0, \nonumber \\
      &\vdots \nonumber \\
      \mathcal{L}_{\mathsf{g}_{\upphi}} \mathcal{L}^{n-2}_{\mathsf{f}_{\upphi}}\mathsf{c}_{\upphi}(\mathrm{x})&=\theta_{n-2}^{'}(\mathrm{x})\, \mathsf{g}_{\upphi}(\mathrm{x})=0,
\end{align*}
which can be expressed in the following generalized form:
\begin{align*}
    \theta_{p}^{'}(\mathrm{x}) \, \mathsf{g}_{\upphi}(\mathrm{x})=0, \forall \, p\in \{0,1,2,\cdots,n-2\}.  
\end{align*}
Hence, we obtain optimization problem~\eqref{3.2}.
    Solving the constrained optimization problem and replacing the components of the traditional controller (\ref{dif1.1}) with the neural Lie derivatives we obtain the following feedback linearized controller
    \begin{align*} 
    \mathsf{u}(\mathrm{x})= - \dfrac{\theta_{n}(\mathrm{x})+k_{n-1}\theta_{n-1}(\mathrm{x})+\cdots k_{0}\theta_{0}(\mathrm{x})-V\mathrm{w}}{\theta_{n-1}^{'}(\mathrm{x})\mathsf{g}_{\upphi}(\mathrm{x})},
\end{align*}
where the identity $\theta_{n-1}^{'}(\mathrm{x})\mathsf{g}_{\upphi}(\mathrm{x})\neq 0,$ holds for $\mathrm{x}\in \mathscr{R}^{n}.$      
        \end{proof}
\subsection{Constraint Handling via the Augmented-Lagrangian Method}
The proposed optimization \eqref{3.2} problem is multidimensional and highly nonconvex owing to the large number of neural-network parameters and the imposed equality constraints. Standard gradient-based training updates the network parameters primarily to reduce the prediction loss and does not preserve feasibility with respect to the constraints. Contrarily, direct solution of the constrained optimization problem are computationally demanding due to the number of parameters. To address this, we propose the augmented-Lagrangian method which embeds the constraint residuals into the training objective through Lagrange-multiplier and quadratic-penalty terms. The gradients of the resulting augmented objective can be readily evaluated using standard auto-differentiation functions and subsequently used for gradient-based optimization. This enables the network parameters to be updated toward minimization prediction-loss while satisfying the imposed constraints.
Let the constrained problem \eqref{3.2} be written as
\begin{align} \label{t1}
\min_{\upphi}\  \mathscr{L}(\upphi) \ \text{s.t.} \ \psi_{p,k}(\upphi)\!=\!0, \ p\!=\!0,\!\ldots\!,n\!-\!2, \  k\!=\!1,\!\ldots\!,m,
\end{align}
where $\psi_{p,k}(\upphi)=\theta_p'(\mathrm{x}_k)\,\mathsf{g}_{\upphi}(\mathrm{x}_k)$ are equality constraints indexed by both sample $k$ and structural index $p$ whose collection determines the feasible set. We solve \eqref{t1} via the augmented-Lagrangian method
\begin{equation}
\mathscr{L}_{\rho}(\upphi,\nu)
\!=\!
\mathscr{L}(\upphi)
\!+\!
\sum_{k=1}^{m}\sum_{p=0}^{n-2}\nu_{p,k}\,\psi_{p,k}(\upphi)
\!+\!
\frac{\rho}{2}\sum_{k=1}^{m}\sum_{p=0}^{n-2}\big(\psi_{p,k}(\upphi)\big)^2
\label{eq:augmented_lagrangian}
\end{equation}
where $\nu_{p,k}\in\mathscr{R}$ denotes the Lagrange multiplier associated with the equality constraint $\psi_{p,k}(\upphi)=0$, and $\rho>0$ is the penalty parameter. The second term on the right hand-side of \eqref{eq:augmented_lagrangian} corresponds to the classical Lagrangian contribution, whereas the quadratic term penalizes violations of the equality constraints.  The original constrained problem is thereby replaced by the unconstrained optimization problem
\begin{align} \label{lm}
 \min_{\upphi}\mathscr{L}_{\rho}(\upphi,\nu).   
\end{align}
The solution of \eqref{lm} is obtained by an iterative scheme, see Algorithm \ref{alg:AL_training_loops}. Let $i$ denote the outer iteration index and for fixed multipliers $\nu^{i}$ and penalty parameter $\rho_i$, the neural network parameters are updated by solving the subproblem
\begin{equation}
\upphi^{i+1}
\approx
\underset{\upphi}{\arg\min}\;
\mathscr{L}_{\rho_i}(\upphi,\nu^{i}),
\end{equation}
using a gradient-based optimizer. The multipliers and penalty parameters are subsequently updated according to
\begin{align*}
 \nu_{p,k}^{\,i+1}
\!&=\!
\nu_{p,k}^{\,i}
\!+\!
\rho_i\,\psi_{p,k}(\upphi^{i+1}),
\,
p\!=\!0,\!\dots\!,n\!-\!2,\,
k\!=\!1,\!\dots\!,m,\\  
\rho_{i+1}
&\!=\!
\min\left\{\beta\rho_i,\rho_{\max}\right\},
\, \beta\!>\!1,
\end{align*}
where $\rho_{\max}$ limits excessive growth of the penalty parameter. Thus, each outer iteration consists of minimizing the augmented objective with respect to the MLP parameters, followed by updates of the multipliers and the penalty parameter. Repeating this procedure progressively reduces the identification error while enforcing the constraints.

\begin{algorithm}[!t]
\caption{Augmented-Lagrangian Training.}
\label{alg:AL_training_loops}
\begin{algorithmic}

\STATE \textbf{Input:}
$\mathrm{x}_0,T,\Delta t,\mathsf{u}(t),\rho_0,
\beta,\rho_{\max},N_{\mathrm{out}},N_{\mathrm{in}}$

\STATE \textbf{Output:} Learned parameters $\upphi$

\STATE Simulate
$\dot{\mathrm{x}}=f(\mathrm{x})+g(\mathrm{x})\mathsf{u}(t)$,
$\mathsf{y}=c(\mathrm{x})$

\STATE Construct
$D=\{(\mathrm{x}_k,\mathsf{u}_k,
\dot{\mathrm{x}}_k,\mathsf{y}_k)\}_{k=1}^{m}$

\STATE Initialize $\upphi$, $\nu_{p,k}\gets0$, and $\rho\gets\rho_0$

\FOR{$i=1,\ldots,N_{\mathrm{out}}$}

    \FOR{$j=1,\ldots,N_{\mathrm{in}}$}

        \STATE Update $\upphi$ by minimizing
        \[
        \!\!\!\!\!\!\!\!\!\!\!\!\!\mathscr{L}_{\rho}(\upphi,\nu)
        \!=\!
        \mathscr{L}(\upphi)
        \!+\!
        \sum_{k=1}^{m}\!\sum_{p=0}^{n-2}
        \!\nu_{p,k}\psi_{p,k}(\upphi)
        \!+\!
        \frac{\rho}{2}
        \sum_{k=1}^{m}\!\sum_{p=0}^{n-2}
        \!(\psi_{p,k}(\upphi))^{2}
        \]

    \ENDFOR

    \STATE Evaluate
    $\psi_{p,k}(\upphi)
    =
    \theta_{p}'(\mathrm{x}_k)
    \mathsf{g}_{\upphi}(\mathrm{x}_k)$

    \STATE Update
    $\nu_{p,k}
    \gets
    \nu_{p,k}+\rho\psi_{p,k}(\upphi)$

    \STATE Update
    $\rho\gets\min\{\beta\rho,\rho_{\max}\}$

\ENDFOR

\STATE Verify
$\theta_{n-1}'(\mathrm{x}_k)
\mathsf{g}_{\upphi}(\mathrm{x}_k)\neq0$,
$k=1,\ldots,m$

\STATE \textbf{return} $\upphi$

\end{algorithmic}
\end{algorithm}

\section{Practical Closed-loop Stability Analysis}     
This section establishes a practical closed-loop stability result for the proposed data-driven feedback-linearizing controller in the presence of neural approximation errors. Using bounded identification errors together with the universal function approximation property (UFAP) of neural networks \cite{h4}, we first derive finite bounds on the errors of the learned feedback-linearization quantities. These bounds are then propagated into the closed-loop dynamics to characterize the resulting perturbations and establish practical stability.
\begin{theorem}
\label{thm:bounded-neural-tracking}
Consider a data-driven nonlinear system of the form \eqref{6}, in a compact operation domain $\mathcal{D} \subseteq \mathscr{R}^{n}.$ Assume the system has relative degree $r=n,$ and the corresponding controller is given in \eqref{c1}. If the MLPs are sufficiently expressive, i.e. fulfill the UFAP, and the controller gains are selected such that the resulting characteristic polynomial is
Hurwitz, then the closed-loop tracking error is uniformly ultimately bounded. Consequently, the learning-based feedback-linearized closed-loop system is practically stable.
\end{theorem}
\begin{proof}
Consider the nonlinear SISO system \eqref{1} with relative degree $r\!=\!n$. Differentiating the output $n$ times yields
\begin{equation} \label{s0}
\mathsf{y}^{(n)}(\mathpzc{t})
=
\alpha(\mathrm{x})
+
\beta(\mathrm{x})\,\mathsf{u}(\mathpzc{t}),
\end{equation}
with $\alpha(\mathrm{x})=\mathcal{L}_{\mathsf{f}}^{\,n}\mathsf{c}(\mathrm{x})$, $\beta(\mathrm{x})=\mathcal{L}_{\mathsf{g}}
\mathcal{L}_{\mathsf{f}}^{\,n-1}\mathsf{c}(\mathrm{x})$.
If the unknown nonlinear mapping is approximated by neural network \eqref{6}, then 
\[
\begin{aligned}
\alpha_{\upphi}(\mathrm{x})
&=
\mathcal{L}_{\mathsf{f}_{\upphi}}^{\,n}
\mathsf{c}_{\upphi}(\mathrm{x})
=
\theta_n(\mathrm{x}),
\\
\beta_{\upphi}(\mathrm{x})
&=
\mathcal{L}_{\mathsf{g}_{\upphi}}
\mathcal{L}_{\mathsf{f}_{\upphi}}^{\,n-1}
\mathsf{c}_{\upphi}(\mathrm{x})
=
\theta_{n-1}'(\mathrm{x})
\mathsf{g}_{\upphi}(\mathrm{x}),
\end{aligned}
\]
where $\theta_0(\mathrm{x}),\cdots,\theta_{n-1}(\mathrm{x}),\theta_n(\mathrm{x})$ are defined as in Theorem \ref{main}. The data-driven neural feedback-linearizing controller \eqref{c1} can be written as
\begin{equation}\label{s1}
\mathsf{u}(\mathpzc{t})
=
\frac{
V\mathrm{w}
-
\alpha_{\upphi}(\mathrm{x})
-
k_{n-1}\theta_{n-1}(\mathrm{x})
-\cdots
-
k_0\theta_0(\mathrm{x})
}
{\beta_{\upphi}(\mathrm{x})},
\end{equation}
where the coefficients
$k_0,\ldots,k_{n-1}$ are selected such that the resulting equation
$
s^n
+
k_{n-1}s^{n-1}
+\cdots
+
k_1s
+
k_0
$
is Hurwitz.
Substituting \eqref{s1} in \eqref{s0}, we obtain
\begin{align} \label{s2}
\mathsf{y}^{(n)}
&=\alpha(\mathrm{x})+\frac{\beta(\mathrm{x})}{\beta_{\upphi}(\mathrm{x})}\big[V \mathrm{w}-\alpha_{\upphi}(\mathrm{x})-\sum_{p=0}^{n-1}k_p\theta_p(\mathrm{x})\big].
\end{align}
Adding
$\sum_{p=0}^{n-1}k_p\mathsf{y}^{(p)}$
to both sides of \eqref{s2}, we obtain
\begin{align} \label{s3}
\mathsf{y}^{(n)}+\sum_{p=0}^{n-1}k_p\mathsf{y}^{(p)}&=
\alpha(\mathrm{x})+\frac{\beta(\mathrm{x})}{\beta_{\upphi}(\mathrm{x})}\big[V \mathrm{w}-\alpha_{\upphi}(\mathrm{x})-\sum_{p=0}^{n-1}k_p\theta_p\big]\nonumber\\
& \quad +
\sum_{p=0}^{n-1}
k_p\mathsf{y}^{(p)} = V \mathrm{w}
+
d(\mathpzc{t},\mathrm{x}),
\end{align}
where
\begin{equation}\label{s4}
\begin{aligned}
d(\mathpzc{t},\mathrm{x})
=&
\alpha(\mathrm{x})
-
\alpha_{\upphi}(\mathrm{x})
+
\sum_{p=0}^{n-1}
k_p
\big[
\mathsf{y}^{(p)}
-
\theta_p(\mathrm{x})
\big]
\\
&+
\frac{
\beta(\mathrm{x})
\!-\!
\beta_{\upphi}(\mathrm{x})
}{
\beta_{\upphi}(\mathrm{x})
}
\big[
V \mathrm{w}
\!-\!
\alpha_{\upphi}(\mathrm{x})
\!-\!
\sum_{p=0}^{n-1}
k_p\theta_p(\mathrm{x})
\big].
\end{aligned}
\end{equation}
Therefore,
\begin{equation}\label{s5}
\mathsf{y}^{(n)}
+
k_{n-1}\mathsf{y}^{(n-1)}
+\cdots
+
k_1\dot{\mathsf{y}}
+
k_0\mathsf{y}
=
V \mathrm{w}
+
d(\mathpzc{t},\mathrm{x}).
\end{equation}
The term $d(\mathrm{x},\mathpzc{t})$ denotes the residual perturbation induced by the mismatch between the true Lie derivatives and the identified Lie derivatives after applying the data-driven feedback-linearizing control law. As the true system mappings and their neural approximations are continuous on the compact operating domain $\mathcal{D}$, it follows from the UFAP of sufficiently expressive MLP, that the corresponding modeling errors are bounded on $\mathcal{D}$. Thus, there exist finite positive constants $\varepsilon_f,\varepsilon_g,\varepsilon_c$ such that
\[
\|\mathsf{f}-\mathsf{f}_{\upphi}\|\leq\varepsilon_f,\quad
\|\mathsf{g}-\mathsf{g}_{\upphi}\|\leq\varepsilon_g,\quad
\|\mathsf{c}-\mathsf{c}_{\upphi}\|\leq\varepsilon_c,
\quad \forall\,\mathrm{x}\in\mathcal{D}.
\]
Since neural networks employ smooth sigmoid activation functions and the
learned Lie derivatives are also continuously
differentiable on $\mathcal{D}$, every term on the right-hand
side of \eqref{s4} is continuously bounded on the
compact operating region and there exists positive constants $\varepsilon_{\alpha},\varepsilon_{\theta}, \varepsilon_{\beta}$ such that
\[
\begin{aligned}
\left|
\alpha(\mathrm{x})
-
\alpha_{\upphi}(\mathrm{x})
\right|
&\le
\varepsilon_{\alpha},
\\
\max_{0\leq p\leq n-1}
\left| 
\mathsf{y}^{(p)}(\mathpzc{t})
-
\theta_p(\mathrm{x})
\right|
&\le
\varepsilon_{\theta}
\\
\left|
\beta(\mathrm{x})
-
\beta_{\upphi}(\mathrm{x})
\right|
&\le
\varepsilon_{\beta}.
\end{aligned}
\]

By assumption $r=n$, there exists a constant $\delta>0$ such that $\inf\limits_{\mathrm{x}\in\mathcal{D}}
\left|\beta_{\upphi}(\mathrm{x})\right|
\geq \delta>0.$
Define $N_{\upphi}(\mathpzc{t},\mathrm{x})
=
\mathscr{V}\,\mathrm{w}
-
\alpha_{\upphi}(\mathrm{x})
-
\sum_{p=0}^{n-1}
k_p\theta_p(\mathrm{x}),$ on a compact domain
$\left|
N_{\upphi}(\mathpzc{t},\mathrm{x})
\right|
\le
\varepsilon_{N}.$
Applying the triangle inequality and using the above bounds, we obtain
\[
\left|
d(\mathpzc{t},\mathrm{x})
\right|
\le
\varepsilon_{\alpha}
+
\sum_{p=0}^{n-1}
|k_p|
\varepsilon_{\theta}
+
\frac{\varepsilon_{\beta}}{\delta}
\varepsilon_{N}=\bar d.
\]
Hence, we get, $\left|d(\mathpzc{t},\mathrm{x})\right|\le\bar d$, i.e. bounded neural modeling errors imply bounded tracking errors. Consequently, the learning based feedback-linearized closed-loop
system is practically stable. If the
neural representation is exact, then
$\bar d=0$.
\end{proof}

\section{Armature controlled DC motor}
  In this study, a nonlinear armature-controlled DC motor is considered to validate the derived results, as illustrated in Fig. \ref{fig:dc}.  Its main components are the armature, stator, commutator, and brushes, which collectively convert the applied DC electrical energy into rotational mechanical motion. 

\begin{figure}[b]
    \centering
    \includegraphics[width=.40 \textwidth]{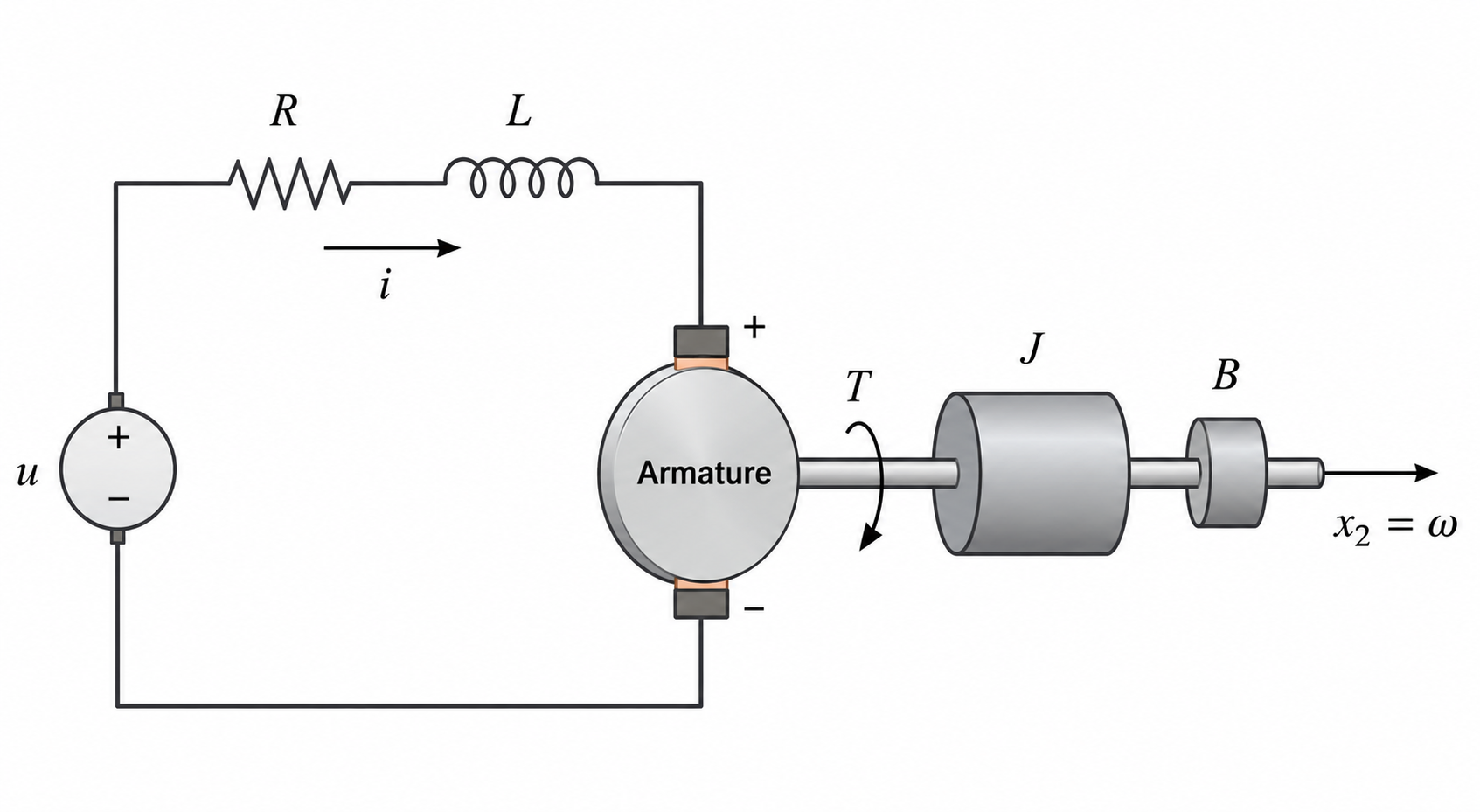}
    \caption{Schematic representation of an armature-controlled DC motor.}
    \label{fig:dc}
\end{figure}
The armature voltage is treated as the control input, the two state variables are armature current $\mathrm{x}_{1}$ and the angular speed of the motor shaft $\mathrm{x}_{2}$. The focus is on speed regulation because the motor speed can be directly controlled by varying the armature voltage, which changes the armature current and the generated torque. Considering nonlinear friction, the DC motor dynamics are expressed as
\begin{align}
\dot{\mathrm{x}}_{1}
&= -\frac{R}{L}\mathrm{x}_{1}
   -\frac{K_b}{L}\mathrm{x}_{2}
   +\frac{1}{L}\mathsf{u}, \nonumber
\\[0.3em]
\dot{\mathrm{x}}_{2}
&= \frac{K_T}{J}\mathrm{x}_{1}
   -\frac{B}{J}\mathrm{x}_{2}
   -\frac{\alpha}{J}\mathrm{x}_{2}^{3},\nonumber
\\[0.3em]
\mathsf{y}
&= \mathrm{x}_{2},
\label{dc1}
\end{align}
\noindent
where \(J\), \(B\), \(R\), and \(L\) denote the rotor moment of inertia, damping coefficient, armature resistance, and armature inductance. 
Further, \(K_b\), \(K_T\), and \(\mathsf{u}\) represent the back-emf constant, motor torque constant and applied armature voltage, respectively. The parameters are given as \(R=2\), \(L=0.5\), \(K_b=0.1\), \(K_T=0.1\), \(J=0.02\), \(B=0.1\), and \(\alpha=0.01\).
  
\subsubsection{Data Generation }
The training dataset is generated by simulating continuous system trajectories initiated from randomized initial conditions. The system is persistently excited by pseudo-random binary sequence  control input $\mathsf{u}_k  \in \{-1,+1\}$ while the state trajectories evolve within the operating region $[-1,2] \times [-1,2]$. Each trajectory is simulated over a time horizon of $T=5.0$ s with a sampling interval of $\Delta t=0.01$ s, where the state derivatives are evaluated from the nonlinear plant model and the states are propagated using the explicit forward-Euler integration scheme. Consequently, each training sample consists of the tuple
$
\left\{
\mathrm{x}_{k},
\mathsf{u}_{k},
\dot{\mathrm{x}}_{k},
\mathsf{y}_{k}
\right\},
$
resulting in an aggregate training dataset of $m=2,000$ samples. 

\par To emulate realistic process and measurement uncertainties, zero-mean Gaussian noise with a standard deviation of $\sigma=0.2$ was introduced during the data generation stage. Specifically, the noise was superimposed on both the numerically propagated state trajectories and the corresponding state derivative observations, resulting in substantially corrupted training data. The augmented-Lagrangian training was performed for \(N_{\mathrm{out}}=25\) outer iterations, with \(N_{\mathrm{in}}=40\) inner epochs per outer iteration, using \(\rho_{0}=1.0\), \(\beta=1.5\), \(\rho_{\max}=50\), zero-initialized Lagrange multipliers, and an output-loss weighting factor of \(\lambda=0.5\). The neural model was trained using Adam optimizer with a learning rate $3 \times 10^{-3}$. 

\subsubsection{Performance Evaluation of neural approximation}
\par Figs. \ref{fig:state_positions} and \ref{fig:state_derivatives_noisy} compare the noisy state trajectories and their corresponding state derivatives with the responses predicted by the learned neural model for four representative initial conditions. The close agreement between the noisy measurements and the predicted trajectories demonstrates that the proposed framework accurately captures both state evolution and underlying system dynamics despite the presence of significant measurement noise. Further, Fig. \ref{fig:loss_matrix_dashboard}, shows the progressive reduction of the augmented Lagrangian training loss, with the constraint violation remaining consistently small throughout training.
\begin{figure}[!t]
    \centering
    \includegraphics[width=0.40 \textwidth]{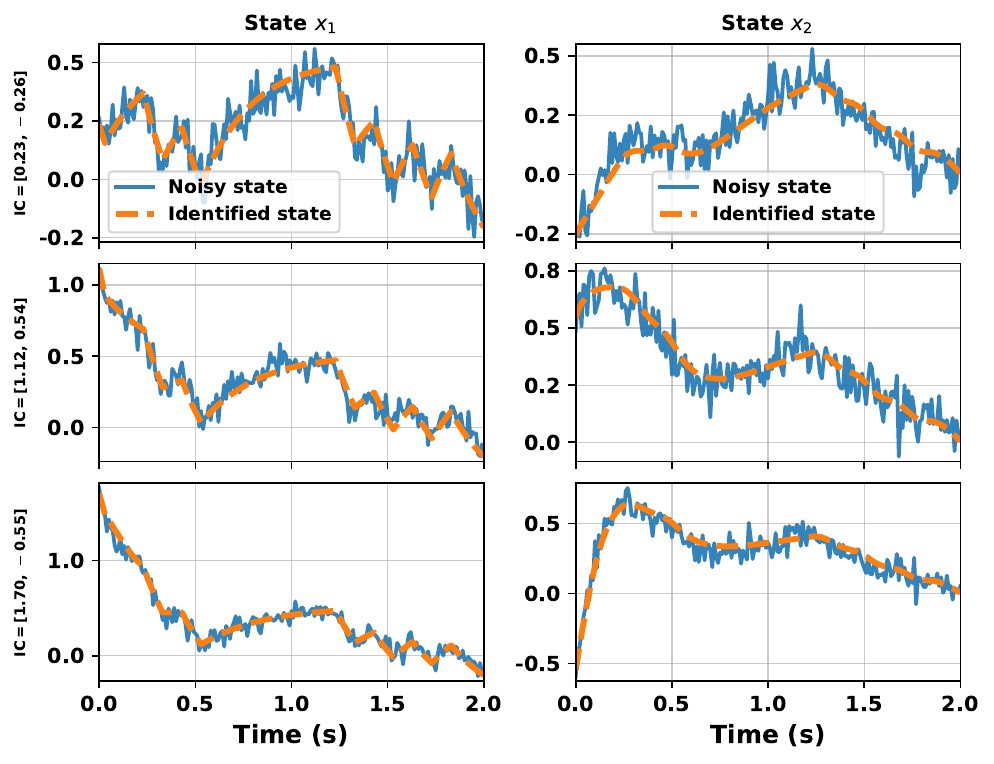}
    \caption{Validation simulation evaluating true continuous state trajectories against the trained model predictions across four distinct initial state configurations. The near overlap between the true and predicted curves also suggests good generalization and a high-quality fit in both transient and steady-state regions.}
    \label{fig:state_positions}
\end{figure}

\begin{figure}[!t]
    \centering
    \includegraphics[width=.45 \textwidth]{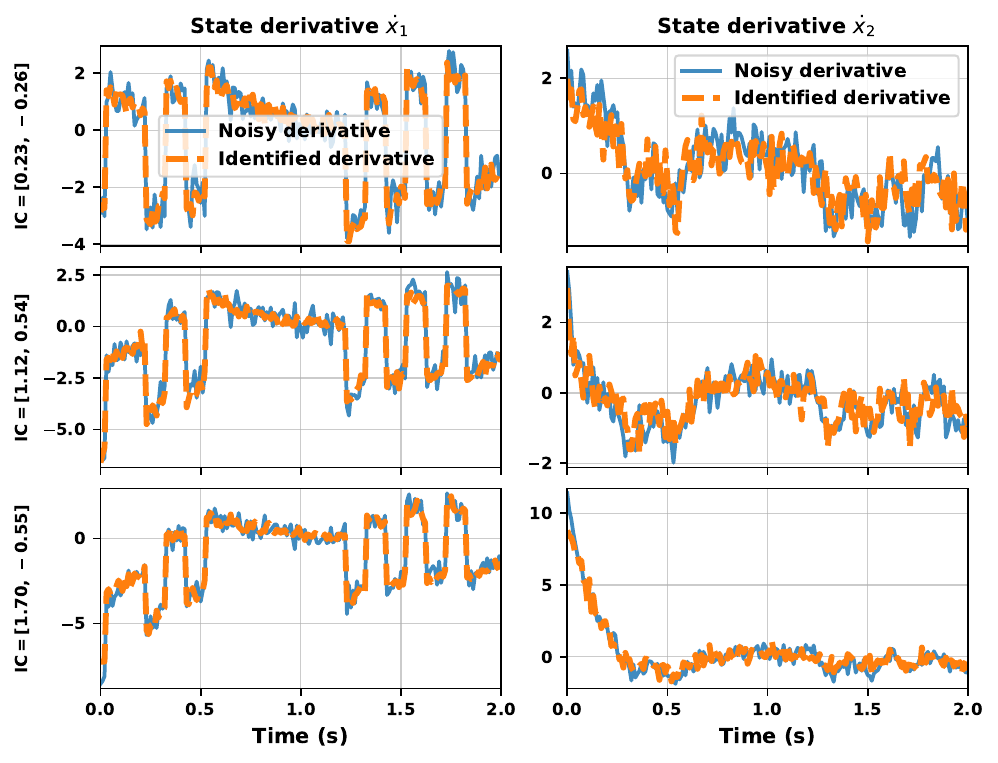}
    \caption{Comparison profiles between the true analytical state derivatives and the neural network vector field predictions ($\dot{\mathrm{x}}_1, \dot{\mathrm{x}}_2$). The learned state-derivative profiles closely follow the true derivatives across all tested initial conditions, showing that the neural model accurately captures the local dynamics of the system. The small deviations visible only in a few early transient portions indicate minor approximation errors, while the overall overlap confirms strong derivative-level consistency.}
    \label{fig:state_derivatives_noisy}
\end{figure}
\begin{figure}[!t]
    \centering
    \includegraphics[width=.40 \textwidth]{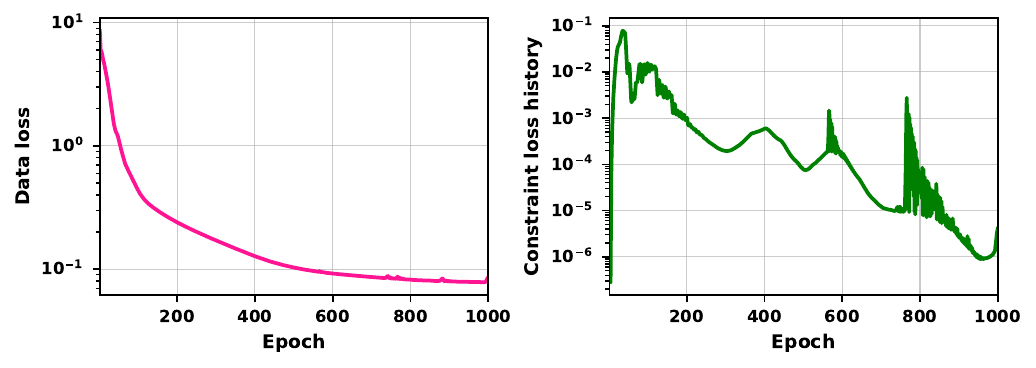}
    \caption{Training Convergence of data loss and constraint residual:  Random initialization initially produces a small constraint residual. As the network parameters are adjusted to improve data fitting, the residual temporarily increase, after which the augmented-Lagrangian terms progressively enforce the constraint and drive the residual toward zero.}
    \label{fig:loss_matrix_dashboard}
\end{figure}

\subsubsection{Data-driven Feedback control}
To evaluate the closed-loop performance of the learned neural model representing the DC motor, the desired closed loop poles are selected at $-3, -5$, resulting in the gain $k_0=15,$ and $ k_1=8$. The resulting data-driven feedback control is given as
\begin{equation}
\mathsf{u}(\mathrm{x})
\!=\!
\frac{\!-\theta_{2}(\mathrm{x})
\!+\!\ddot{r}(\mathpzc{t})\!-\!8\left[\theta_{1}(\mathrm{x})\!-\!\dot{r}(\mathpzc{t})\right]
\!-\!15\left[\theta_{0}(\mathrm{x})\!-\!r(\mathpzc{t})\right]
}
{\theta_{1}^{'}(\mathrm{x})\mathsf{g}_{\upphi}(\mathrm{x})}\nonumber.
\end{equation}
The state trajectories are randomly initialized away from the equilibrium to evaluate the transient regulation performance, while the corresponding control input is monitored to verify that the imposed control bounds are respected. The closed-loop tracking performance is evaluated  with the reference trajectory
$r(\mathpzc{t})=1-e^{-\mathpzc{t}}$, whose derivatives are
$\dot{r}(\mathpzc{t})=e^{-\mathpzc{t}}$ and
$\ddot{r}(\mathpzc{t})=-e^{-\mathpzc{t}}$.
The exponential reference provides a smooth transition to the desired equilibrium with continuous first and second derivatives, making it suitable for the feedback-linearizing controller while avoiding the excessive control effort associated with an ideal step input. To compensate for accumulated tracking errors and suppress steady-state offsets, an integral error compensation is incorporated into the feedback-linearizing controller.
The simulation results presented in the Figs. \ref{fig:noisy1} and \ref {fig:noisy2} demonstrate the closed-loop performance of the proposed data-driven neural feedback-linearizing controller under noisy operating conditions. As shown in Fig.~\ref{fig:noisy1} the state trajectories converge toward the desired equilibrium while the corresponding control input remains bounded. Fig.~\ref {fig:noisy2} illustrates that the output follows the prescribed reference trajectory with a bounded tracking error and a bounded control effort.

\begin{figure}[t!]
    \centering
    \includegraphics[width=0.45\textwidth]{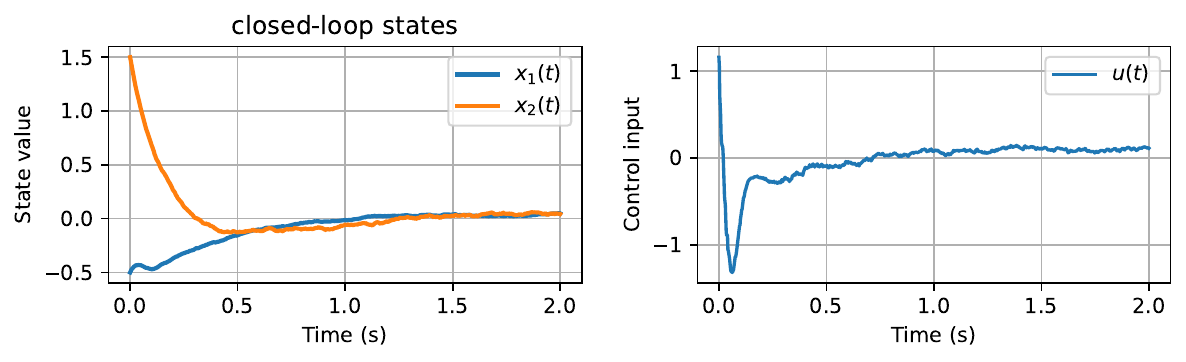}
    \caption{Closed-loop states under noisy operating conditions. The state trajectories converge toward the desired equilibrium while the corresponding control input remains bounded.}
    \label{fig:noisy1}
\end{figure}

\begin{figure}[t!]
    \centering
    \includegraphics[width=0.40 \textwidth]{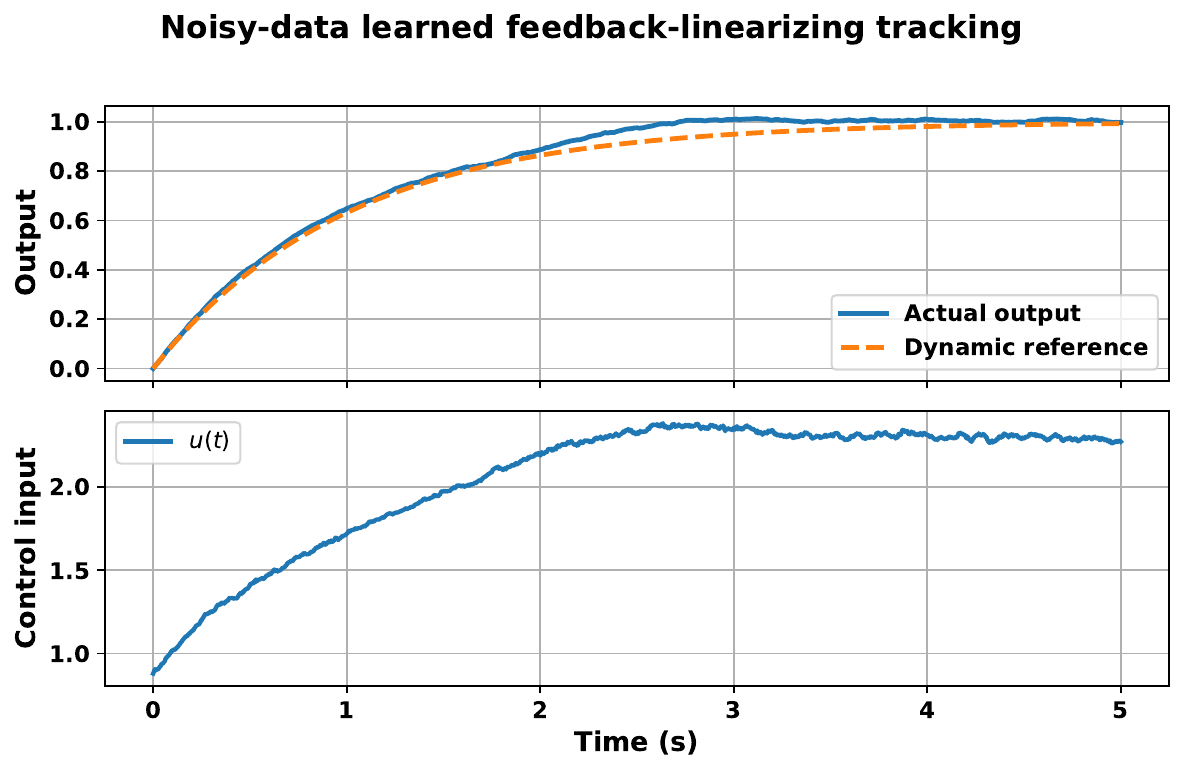}
    \caption{Closed-loop reference tracking performance under noisy operating conditions. The system output follows the reference trajectory with bounded tracking error while maintaining a bounded control input.}
    \label{fig:noisy2}
\end{figure}
\subsubsection{Numerical Verification of Practical Stability}
For the neural model identified from noisy data, Fig.~\ref{fig:error} illustrates the residual modeling error \(d(\mathpzc{t},\bar{\mathrm{x}})\) along the simulated closed-loop trajectories, evaluated according to Theorem~\ref{thm:bounded-neural-tracking}. The maximum residual magnitude is approximately $\bar d \approx 1.04$, such that $\left|d(\mathpzc{t},\bar{\mathrm{x}})\right| \leq \bar d,$ throughout the closed-loop operation. The larger residual bound obtained for the noisy model is attributed to the process and measurement noise introduced during training. The results verify the assumptions of the practical stability theorem and confirm that the learning-based feedback-linearized closed-loop system remains practically stable.
\begin{figure}[htb]
    \centering
    \includegraphics[width=0.40 \textwidth]{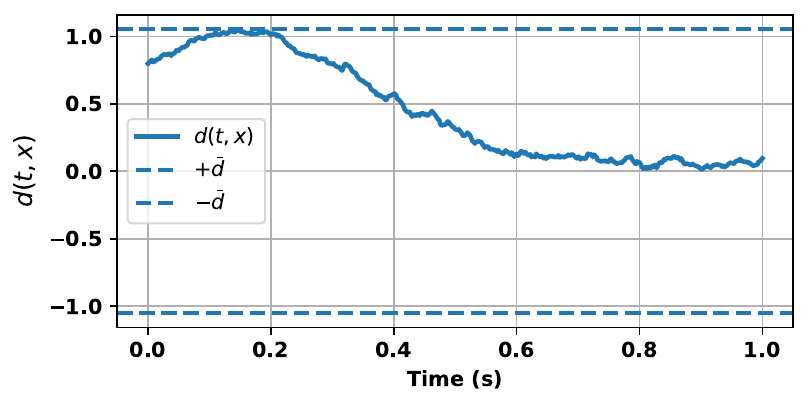}
    \caption{Residual modeling error under noisy operating conditions. The modeling error decreases rapidly during the initial transient and gradually converges to a bounded steady-state value, demonstrating consistent estimation performance despite the presence of noise.}
    \label{fig:error}
\end{figure}

\section{Conclusion}
   This work proposed a data-driven framework for nonlinear system identification and feedback linearization using multilayer perceptron neural networks. The Lie-derivative structure required for feedback linearization was incorporated directly into the learning process through an Augmented-Lagrangian formulation, thereby promoting consistency between the identified model and the structural requirements of the subsequent controller design. The numerical results indicate that the learned model can reproduce the relevant nonlinear dynamics while enabling both stabilization and reference tracking under noisy operating conditions. In addition, the established practical stability result shows that bounded identification errors lead to ultimately bounded tracking errors, providing a direct connection between model accuracy and closed-loop performance. The application to an armature-controlled DC motor further demonstrates the applicability of the proposed approach in the presence of process and measurement uncertainties.  A limitation of the present framework is the assumption that the system possesses a known and uniform relative degree over the considered operating region. Accordingly, the required Lie-derivative constraints are prescribed a priori during training. Future work will therefore focus on jointly identifying the relative-degree structure and enforcing the corresponding feedback-linearization conditions, including the nonvanishing decoupling condition. Further developments will also investigate improved robustness against modeling uncertainty and measurement noise, reduced data requirements, and validation on experimental nonlinear systems.


\section*{References}
\begin{thebibliography}{999}

\bibitem{a1} M. Alsalti, V. G. Lopez, J. Berberich, F. Allgöwer, and M. A. Müller, ``Data-based control of feedback linearizable systems,'' \emph{IEEE Trans. Autom. Control}, vol. 68, no. 11, pp. 7014--7021, 2023.

\bibitem{a2} K. J. {\AA}str\"om and B. Wittenmark, \emph{Adaptive Control}. Reading, MA, USA: Addison-Wesley, 1989.

\bibitem{b1} A. Banaszuk and J. Hauser, ``Approximate feedback linearization: A homotopy operator approach,'' \emph{SIAM J. Control Optim.}, vol. 34, no. 5, pp. 1533--1554, 1996.

\bibitem{b2} S. L. Brunton, J. L. Proctor, and J. N. Kutz, ``Discovering governing equations from data by sparse identification of nonlinear dynamical systems,'' \emph{Proc. Natl. Acad. Sci. USA}, vol. 113, no. 15, pp. 3932--3937, 2016.

\bibitem{d1} A. Dong, A. Starr, and Y. Zhao, ``Neural network-based parametric system identification: A review,'' \emph{Int. J. Syst. Sci.}, vol. 54, no. 13, pp. 2676--2688, 2023.

\bibitem{d2} X. Dai, C. De Persis, N. Monshizadeh, and P. Tesi, ``Data-driven output feedback control of nonlinear systems: Stabilization and robustness,'' \emph{Eur. J. Control}, vol. 87, 2026, Art. no. 101436.


\bibitem{f2} L. Fraile, M. Marchi, and P. Tabuada, ``Data-driven stabilization of SISO feedback linearizable systems,'' arXiv:2003.14240, 2020.


\bibitem{g2} D. Gadginmath, V. Krishnan, and F. Pasqualetti, ``Data-driven feedback linearization using the Koopman generator,'' \emph{IEEE Trans. Autom. Control}, vol. 69, no. 12, pp. 8844--8851, 2024.


\bibitem{h0} H. Hjalmarsson, M. Gevers, S. Gunnarsson, and O. Lequin, ``Iterative feedback tuning: Theory and applications,'' \emph{IEEE Control Syst. Mag.}, vol. 18, no. 4, pp. 26--41, 1998.

\bibitem{h2} J. Hu and K. Hirasawa, ``A method for applying multilayer perceptrons to control of nonlinear systems,'' in \emph{Proc. of the 9th International Conference on Neural Information Processing. IEEE}, vol. 3, 2002, pp. 1267--1271.

\bibitem{h3} A. Hosseinipour and J. Khazaei, ``Sparse identification for data-driven dynamics and impedance modeling of power converters in dc microgrids,'' \emph{IEEE J. Emerg. Sel. Topics Ind. Electron.}, vol. 5, no. 2, pp. 720--732, 2024.

\bibitem{h4} K. Hornik, M. Stinchcombe, and H. White, ``Multilayer feedforward networks are universal approximators,'' \emph{Neural Netw.}, vol. 2, no. 5, pp. 359--366, 1989.

\bibitem{i1} A. Isidori, \emph{Nonlinear Control Systems: An Introduction, 3rd ed.} Berlin, Germany: Springer-Verlag, 1995.


\bibitem{k1} K. S. Narendra and K. Parthasarathy, ``Identification and control of dynamical systems using neural networks,'' \emph{IEEE Trans. Neural Netw.}, vol. 1, no. 1, pp. 4--27, 1990.

\bibitem{k0} M. Karimshoushtari and C. Novara, ``Design of experiments for nonlinear system identification: A set membership approach,'' \emph{Automatica}, vol. 119, 2020, Art. no. 109036.


\bibitem{k3} P. Komarov, F. van Breugel, and J. N. Kutz, ``A taxonomy of numerical differentiation methods,'' arXiv:2512.09090, 2025, doi: 10.48550/arXiv.2512.09090.

\bibitem{l1} P. K. Lakshmi Priya and A. Schwung, ``Data driven feedback linearization of nonlinear control systems via Lie derivatives and stacked regression approach,'' arXiv:2508.13241, 2025.

\bibitem{l2} L. Li, A. Bisoffi, C. De Persis, and N. Monshizadeh, ``Controller synthesis from noisy-input noisy-output data,'' \emph{Automatica}, vol. 183, 2026, Art. no. 112545.

\bibitem{m1} T. Martin and F. Allgöwer, ``Data-driven system analysis of nonlinear systems using polynomial approximation,'' \emph{IEEE Trans. Autom. Control}, vol. 69, no. 7, pp. 4261--4274, 2024.

\bibitem{m2} I. Markovsky and P. Rapisarda, ``Data-driven simulation and control,'' \emph{Int. J. Control}, vol. 81, no. 12, pp. 1946--1959, 2008.

\bibitem{n1} D. H. Nguyen and B. Widrow, ``Neural networks for self-learning control systems,'' \emph{IEEE Control Syst. Mag.}, vol. 10, no. 3, pp. 18--23, 1990.



\bibitem{p2} C. De Persis, D. Gadginmath, F. Pasqualetti, and P. Tesi, ``Feedback linearization through the lens of data,'' \emph{IEEE Trans. Autom. Control}, vol. 71, no. 3, pp. 1630--1643, 2026.

\bibitem{r1} S. Razani, M. Tavakoli-Kakhki, and A. Kalhor, ``Robust data-driven feedback linearization using neural network-based sparse identification of nonlinear dynamics,'' \emph{ISA Trans.}, vol. 176, pp. 67--80, 2026,

\bibitem{s1} J. Sjöberg, H. Hjalmarsson, and L. Ljung, ``Neural networks in system identification,'' \emph{IFAC Proc. Vol.}, vol. 27, no. 8, pp. 359--382, 1994.

\bibitem{s0} J. Sjöberg, Q. Zhang, L. Ljung, A. Benveniste, B. Delyon, P.-Y. Glorennec, H. Hjalmarsson, and A. Juditsky, ``Nonlinear black-box modeling in system identification: A unified overview,'' \emph{Automatica}, vol. 31, no. 12, pp. 1691--1724, 1995.

\bibitem{s2} G. Singh, M. P. Nandakumar, and S. Ashok, ``Adaptive fuzzy-PID and neural network based object tracking using a 3-axis platform,'' in \emph{Proc. 2016 IEEE Int. Conf. Eng. Technol. (ICETECH)}, 2016, pp. 1012--1017.

\bibitem{s3} S. Şahin, ``Learning feedback linearization using artificial neural networks,'' \emph{Neural Process. Lett.}, vol. 44, no. 3, pp. 625--637, 2016.

\bibitem{s4} M. G. Safonov and T. Tsao, ``The unfalsified control concept and learning,'' \emph{IEEE Trans. Autom. Control}, vol. 42, no. 6, pp. 843--847, 1997.

\bibitem{t1} M. Thieffry, A. Hache, M. Yagoubi, and P. Chevrel, ``Identification for control based on neural networks: Approximately linearizable models,'' arXiv:2409.15858, 2024.

\bibitem{u1} J. Umlauft and S. Hirche, ``Feedback linearization based on Gaussian processes with event-triggered online learning,'' \emph{IEEE Trans. Autom. Control}, vol. 65, no. 10, pp. 4154--4169, 2019.

\bibitem{v1} H. J. van Waarde, C. De. Persis, M. K. Camlibel, and P. Tesi, ``Willems's fundamental lemma for state-space systems and its extension to multiple datasets,'' \emph{IEEE Control Syst. Lett.}, vol. 4, no. 3, pp. 602--607, 2020.

 

\bibitem{x1} J. Xie, F. Bonassi, and R. Scattolini, ``Learning control affine neural NARX models for internal model control design,'' \emph{IEEE Trans. Autom. Sci. Eng.}, vol. 22, pp. 8137--8149, 2025.

\bibitem{w1} T. Westenbroek, D. Fridovich-Keil, E. Mazumdar, S. Arora, V. Prabhu, S. S. Sastry, and C. J. Tomlin, ``Feedback linearization for uncertain systems via reinforcement learning,'' in \textit{Proc. IEEE Int. Conf. Robot. Autom. (ICRA)},
2020, pp. 1364--1371.

\bibitem{w2} J. C. Willems, P. Rapisarda, I. Markovsky, and B. L. M. De Moor, ``A note on persistency of excitation,'' \emph{Syst. Control Lett.}, vol. 54, no. 4, pp. 325--329, 2005.

\bibitem{y1} A. Yeşildirek and F. L. Lewis, ``Feedback linearization using neural networks,'' \emph{Automatica}, vol. 31, no. 11, pp. 1659--1664, 1995.

\bibitem{z1} J. G. Ziegler and N. B. Nichols, ``Optimum settings for automatic controllers,'' \emph{Trans. ASME}, vol. 64, pp. 759--768, 1942.

\end{thebibliography}
\end{document}